\documentclass[letterpaper]{article} 
\usepackage{aaai24}  
\usepackage{times}  
\usepackage{helvet}  
\usepackage{courier}  
\usepackage[hyphens]{url}  
\usepackage{graphicx} 
\usepackage{natbib}  
\usepackage{caption} 
\usepackage{algorithm}

\usepackage{algpseudocode}
\usepackage{amsthm}
\usepackage{amsmath}
\DeclareMathOperator*{\argmax}{arg\,max}
\usepackage{amssymb}
\usepackage{mathtools}
\usepackage{booktabs} 
\usepackage{tikz}
\usepackage{comment}
\usepackage{subfig}
\newtheorem{theorem}{Theorem}
\newtheorem{corollary}{Corollary}
\newtheorem{lemma}{Lemma}

\newtheorem{remark}{Remark}

\usepackage{newfloat}
\usepackage{listings}
\DeclareCaptionStyle{ruled}{labelfont=normalfont,labelsep=colon,strut=off} 
\floatstyle{ruled}
\newfloat{listing}{tb}{lst}{}
\floatname{listing}{Listing}
\title{Combinatorial Stochastic-Greedy Bandit}
\author {
    Fares Fourati\textsuperscript{\rm 1},
    Christopher John Quinn\textsuperscript{\rm 2},
    Mohamed-Slim Alouini\textsuperscript{\rm 1},
    Vaneet Aggarwal\textsuperscript{\rm 3, \rm 1}
}
\affiliations {
    \textsuperscript{\rm 1}King Abdullah University of Science and Technology\\
    \textsuperscript{\rm 2} Iowa State University\\
    \textsuperscript{\rm 3} Purdue University \\
    fares.fourati@kaust.edu.sa, cjquinn@iastate.edu, slim.alouini@kaust.edu.sa, vaneet@purdue.edu 
}

\usepackage{bibentry}

\begin{document}

\maketitle

\begin{abstract}
We propose a novel combinatorial stochastic-greedy bandit (SGB) algorithm for combinatorial multi-armed bandit problems when no extra information other than the joint reward of the selected set of $n$ arms at each time step $t\in [T]$ is observed. SGB adopts an optimized stochastic-explore-then-commit approach and is specifically designed for scenarios with a large set of base arms. Unlike existing methods that explore the entire set of unselected base arms during each selection step, our SGB algorithm samples only an optimized proportion of unselected arms and selects actions from this subset. We prove that our algorithm achieves a $(1-1/e)$-regret bound of $\mathcal{O}(n^{\frac{1}{3}} k^{\frac{2}{3}} T^{\frac{2}{3}} \log(T)^{\frac{2}{3}})$ for monotone stochastic submodular rewards, which outperforms the state-of-the-art in terms of the cardinality constraint $k$. Furthermore, we empirically evaluate the performance of our algorithm in the context of online constrained social influence maximization. Our results demonstrate that our proposed approach consistently outperforms the other algorithms, increasing the performance gap as $k$ grows.
\end{abstract}

\section{Introduction}

The stochastic multi-armed bandits (MAB) problem involves selecting an arm in each round $t$ and observing a reward that follows an unknown distribution. The objective is to maximize the accumulated reward within a finite time horizon $T$. Solving the classical MAB problem requires striking a balance between exploration and exploitation. Should the agent try arms that have been explored less frequently to gather more information (exploration), or should it stick to arms that have yielded higher rewards based on previous observations (exploitation)? An extension of the MAB problem is the combinatorial MAB problem, where, instead of choosing a single arm per round, the agent selects a set of multiple arms and receives a joint reward for that set. When the agent only receives information about the reward associated with the selected set of arms, it is known as \textit{full-bandit feedback} or simply \textit{bandit feedback}.
On the other hand, if the agent obtains additional information about the contribution of each arm to the overall reward, it is referred to as \textit{semi-bandit feedback}. The full-bandit feedback setting poses a more significant challenge as the decision-maker has significantly less information than the semi-bandit feedback scenario \citep{fourati2023randomized}. This paper focuses on the first scenario for combinatorial MAB, i.e., the bandit feedback setting.

In recent years, there has been growing interest in studying combinatorial multi-armed bandit problems with submodular\footnote{A set function $f: 2^{\Omega} \rightarrow \mathbb{R}$ defined on a finite set $\Omega$ is considered submodular if it exhibits the property of diminishing returns: for any $A \subseteq B \subset \Omega$ and $x \in \Omega \backslash B$, it holds that $f(A \cup{x})-f(A) \geq f(B \cup{x})-f(B)$.} reward functions \citep{49310, nie2022explore, fourati2023randomized}. The submodularity assumption finds motivation in various real-world scenarios. For instance, opening additional supermarkets in a specific location would result in diminishing returns due to market saturation. As a result, submodular functions are commonly used as objective functions in game theory, economics, and optimization \citep{fourati2023randomized}. Submodularity arises in important contexts within combinatorial optimization, such as graph cuts \citep{goemans1995improved, iwata2001combinatorial}, rank functions of matroids \citep{edmonds2003submodular}, and set covering problems \citep{feige1998threshold}. Some key problems where combinatorial multi-armed bandit problems with submodular reward functions include proposing items with redundant information \citep{Qin2013PromotingDI,  takemori2020submodular}, optimizing client participation in federated learning \citep{balakrishnan2022diverse,fourati2023filfl}, and social influence maximization without knowledge of the network or diffusion model \citep{ wen2017online, li2020online, perrault2020budgeted,agarwal2021stochastic, nie2022explore}.

Similar to previous works \citep{streeter2008online, golovin2014online, niazadeh2021online, agarwal2021dart, agarwal2021stochastic, nie2022explore}, we assume that the reward function is non-decreasing (monotone) in expectation. Without further constraints, the optimal set will contain all the arms in this setup. Thus, we limit the cardinality of the set to $k$. Recently, \citep{nie2022explore,nie2023framework} studied this problem and proposed an explore-then-commit greedy (ETCG) algorithm for this problem with full-bandit feedback and showed a $(1-1/e)$-regret bound of $\tilde{\mathcal{O}}(n^\frac{1}{3}kT^\frac{2}{3})$, where $n$ is the number of arms. The algorithm follows a greedy explore-then-commit approach that greedily adds base arms to a super arm (a subset of base arms) until the cardinality constraint is satisfied. It then exploits this super arm for the remaining time. To determine which base arm to add to the super arm, the remaining arms are sampled $m$ times each (where $m$ is a hyper-parameter), and the arm with the highest average reward is chosen. However, in practical scenarios with many arms, exploring all remaining arms in each iteration may require a significant time and thus is unsuitable for smaller $T$. Therefore, we propose a modified approach where a smaller subset of arms is randomly selected for exploration in each iterative round, and the arm with the highest reward is chosen. 

It is worth noting that a similar random selection-based algorithm has been considered in \citep{mirzasoleiman2015lazier} for the offline setup, providing a $(1-1/e-\epsilon)$-approximation guarantee, where $\epsilon$ determines the reduction in the number of arms selected in each iteration. However, although beneficial for exploration, this sub-selection of arms in each iteration leads to suboptimal approximation guarantees. In this paper, we ask the question: ``{\em Can exploring a subset of arms in each iteration achieve a benefit regarding the $(1-1/e)$-regret bound compared to selecting all remaining arms?}"

We answer this question in a positive. By carefully selecting the parameter $\epsilon$, we achieve a $(1-1/e)$-regret bound of $\tilde{\mathcal{O}}(n^{\frac{1}{3}} k^{\frac{2}{3}} T^{\frac{2}{3}})$. This improvement in regret bound surpasses that of \citep{nie2022explore, nie2023framework} by orders of magnitude in terms of $k$. This improvement is particularly significant for larger values of $k$. Our proposed approach reduces exploration while enhancing expected cumulative $(1-1/e)$-regret performance.

\subsection{ Contributions}

We present the main contributions of this paper as follows:

\begin{itemize}
\item We introduce stochastic-greedy bandit (SGB), a novel technique in the explore-then-commit greedy strategy with bandit feedback, wherein an optimized proportion of remaining arms are randomly sampled in each greedy iteration. More precisely, rather than sampling $(n-i+1)$ arms in greedy iteration $i$, random $(n-i+1)\min\{1, \log(1/\epsilon)/k\}$ arms are chosen for an appropriately selected $\epsilon = {\mathcal O}(n^{\frac{1}{3}}k^{\frac{2}{3}}T^{-\frac{1}{3}}\log(T)^{-\frac{1}{3}})$, which reduces the amount of exploration.

\item We provide theoretical guarantees for SGB by proving that it achieves an expected cumulative $(1-1/e)$-regret of at most $\tilde{\mathcal{O}}(n^{\frac{1}{3}} k^{\frac{2}{3}} T^{\frac{2}{3}})$ for monotone stochastic submodular rewards. This represents an  improvement of $k^{\frac{1}{3}}$ compared to the previous state-of-the-art method \citep{nie2023framework}.

\item We conduct empirical experiments to evaluate the performance of our proposed SGB algorithm compared to the previous state-of-the-art algorithms specialized in monotone stochastic submodular rewards \citep{nie2022explore, nie2023framework}. We specifically focus on the online social influence maximization problem and demonstrate the efficiency of SGB in achieving superior results in terms of cumulative regret. In particular, the results show that the proposed algorithm outperforms the baselines, with the performance gap increasing as $k$ increases. 

\end{itemize}

\subsection{Related Work} \label{sec:related-work}

\begin{table}
\centering
\resizebox{\columnwidth}{!}{%
\begin{tabular}{lcc}%
\toprule
%
%
%
Reference &  Stochastic  & $(1-1/e)$-Regret  \\
\midrule
\cite{streeter2008online}   &    &   $\tilde{\mathcal{O}}(\ n^\frac{1}{3}\ k^2\ T^\frac{2}{3}\ )$  \\
\cite{golovin2014online}    & &   $\tilde{\mathcal{O}}(\ n^\frac{2}{3}\ k^\frac{2}{3}\ T^\frac{2}{3}\ )$  \\
\cite{niazadeh2021online}     &   &   $\tilde{\mathcal{O}}(\ n^\frac{2}{3}\ k\phantom{\frac{3}{3}}\ T^\frac{2}{3}\ )$   \\
 \cite{nie2022explore}  & \checkmark &  $\tilde{\mathcal{O}}(\ n^\frac{1}{3}\ k^\frac{4}{3}\ T^\frac{2}{3}\ )$    \\
 \cite{nie2023framework}  & \checkmark  &  $\tilde{\mathcal{O}}(\ n^\frac{1}{3}\ k\phantom{^\frac{4}{3}}\ T^\frac{2}{3}\ )$    \\
SGB($\epsilon^\star$) (ours)  &  \checkmark &  $\tilde{\mathcal{O}}(\ n^\frac{1}{3}\ k^{\frac{2}{3}} \ T^\frac{2}{3}\ )$    \\
\bottomrule
\end{tabular}
}
\caption{Table of selected related works for submodular monotone function maximization under full-bandit feedback. The second column indicates the stochastic or adversarial setting, and the last column gives the expected cumulative $(1-1/e)$-regret bounds. The notation $\tilde{\mathcal{O}}(\cdot)$ drops the $\log$ terms.}
\label{tab:related-work}
\end{table}

This section discusses the works closely related to the problem we are investigating. Multi-armed bandits have been considered in two different settings: \textit{adversarial setting}, where an adversary produce a reward succession that may be affected by the agent's prior decisions \citep{auer2002nonstochastic}, and a \textit{stochastic setting}, where the reward of each action is randomly drawn from a specific distribution, as described in \citep{auer2002finite}. In this work, we focus on stochastic reward functions. Standard multi-armed bandits find adversarial settings more challenging, and the outcome can be immediately used as one workable method for the stochastic scenario \citep{lattimore2020bandit}. However, this is different for submodular bandits. While adversarial environment in adversarial bandits selects a series of submodular functions $\left\{f_1,\cdots, f_T \right\}$ \citep{streeter2008online, golovin2014online, pmlr-v75-roughgarden18a, 49310}, in the submodular stochastic setting the realizations of the stochastic function in the problem we define need not be submodular, it only needs to be submodular in expectation  \citep{fourati2023randomized}, meaning the stochastic setting is not a particular case of the adversarial setting.

Submodular maximization has been proven to be NP-hard. Even achieving an approximation ratio of $\alpha\in (1-1/e,1]$ under a cardinality constraint with access to a monotone submodular value oracle is also NP-hard \citep{nemhauser1978analysis,feige1998threshold,feige2011maximizing}. However, \citep{nemhauser1978analysis} proposed a simple greedy $(1-1/e)$-approximation algorithm for monotone submodular maximization under a cardinality constraint. Therefore, the best approximation ratio for monotone submodular objectives with a polynomial time algorithm is $1-1/e$. Thus, we study $(1-1/e)$-regret combinatorial MAB algorithms in this paper.

Table \ref{tab:related-work} enumerates related combinatorial works with monotone submodular reward function under bandit feedback for both adversarial and stochastic settings. The table summarizes that the proposed approach achieves the state-of-the-art $(1-1/e)$-regret result. Even though we consider stochastic submodular rewards, full-bandit feedback has been studied for non-submodular rewards, including linear reward functions \citep{dani2008stochastic,rejwan2020top} and Lipschitz reward functions \cite{agarwal2021dart,agarwal2021stochastic}. In these works, the optimal action (best set of $k$ arms) is to use the $k$ individually best arms; that property does not hold for submodular rewards. Further, non-monotone submodular functions with bandit feedback without cardinality constraint have been studied in \citep{fourati2023randomized}, where $\frac{1}{2}$-regret is derived. However, this algorithm cannot be directly applied to our setup since it lacks a cardinality constraint.

Recently, \citep{nie2023framework} provided a framework that adapts offline algorithms for combinatorial optimization with a robustness guarantee to online algorithms with provable regret guarantees. We could use this framework for the offline approximation algorithm described in \citep{mirzasoleiman2015lazier} for the problem. At the same time, we note that such an approach will result in $(1-1/e-\epsilon)$-approximation because the offline algorithm has $(1-1/e-\epsilon)$ guarantee. Thus, exploring only a subset of arms in each iteration and achieving a $(1-1/e)$-regret is non-trivial and requires careful analysis of the algorithm, which is done in this paper. 

\section{Problem Statement}
\label{prob_state}

In this section, we present the problem formally. We denote $\Omega$, the ground set of base arms which includes $n$ base arms. We consider decision-making problems with a fixed time horizon $T$, where the agent, at each time step $t$, chooses an action $S_t \subseteq \Omega$, with maximum cardinality constraint $k$. Let $\mathcal{S}=\{S | S \subseteq \Omega  \text{ and } |S|\leq k\}$ represent the set of all permitted subsets at any time step. 

After deciding the action $S_t$, the agent acquires reward $f_t(S_t)$. We assume the reward $f_t$ is stochastic, bounded in $[0,1]$, i.i.d. conditioned on a given action, submodular in expectation\footnote{A stochastic set function $f: 2^{\Omega} \rightarrow \mathbb{R}$ defined on a finite set $\Omega$ is considered submodular in expectation if for all $A \subseteq B \subset \Omega$, and $x \in \Omega \backslash B$, we have,
 \begin{equation}
     \mathbb{E}[f\left(A \cup\{x\}\right)]-\mathbb{E}[f(A)] \geq \mathbb{E}[f(B \cup\{x\})]-\mathbb{E}[f(B)].
 \end{equation}}, and monotonically non-decreasing in expectation\footnote{A stochastic set function $f: 2^{\Omega} \rightarrow \mathbb{R}$ is called non-decreasing in expectation if for any $A \subseteq B \subseteq \Omega$ we have  $\mathbb{E}[ f(A)] \leq \mathbb{E}[f(B)]$. }. The goal of the agent is to maximize the cumulative reward 
$\sum_{t=1}^Tf_t(S_t)$. Define the expected reward function as 
$f(S) = \mathbb{E}[f_t(S)]$, hence 
$S^{\star}=\argmax
_{S:|S|\leq k}f(S)$
denote the optimal solution in expectation. One common metric to measure the algorithm's performance is to compare the learner to an agent that has and always chooses the optimal set in expectation $S^{\star}$. 

The best approximation ratio for monotone-constrained submodular objectives with a polynomial time algorithm is $1-1/e$ \citep{nemhauser1978analysis}. Therefore, we compare the learner's cumulative reward to $(1-1/e)Tf(S^{\star})$, and we denote the difference as the ($1-1/e$)-regret, which is defined as follows
\begin{equation}
    \mathcal{R}(T) = (1-\frac{1}{e})Tf(S^{\star}) - \sum_{t=1}^T f_t(S_t). \label{eq:reg:1e}
\end{equation}  

Note that the ($1-1/e$)-regret $\mathcal{R}(T)$ is random and depends on the subsets chosen. In this work, we focus on minimizing the expected cumulative $(1-1/e)$-regret 
\begin{align}
    \hspace{-.1cm}\mathbb{E}[\mathcal{R}(T)] = (1-\frac{1}{e})Tf(S^{\star}) - \mathbb{E}\left[\sum_{t=1}^T f_t(S_t)\right],\label{eq:reg:exp1e}
\end{align} 
where the expectation is over both the environment and the sequence of actions.

\section{Proposed Algorithm}
\label{sec:proposedalg}

This section presents our proposed combinatorial stochastic-greedy bandit (SGB) algorithm that applies our optimized stochastic-explore-then-commit approach. We provide its pseudo-code in Algorithm~\ref{alg:sgb}. 

\begin{algorithm}
\caption{SGB}
\label{alg:sgb}
\begin{algorithmic}
\Require ground set $\Omega$, horizon $T$ , cardinality $k$ \\
$S^{(0)} \leftarrow \emptyset$, $n \leftarrow|\Omega|$ \\
$m \leftarrow \left \lceil  \left(\frac{kT }{2n \sqrt{\log (T)}}\right)^{\frac{2}{3}} \right \rceil$, $\epsilon \leftarrow \left(\frac{n k^2 }{4 T \log (T)}\right)^{\frac{1}{3}}$\\
$\beta \leftarrow \frac{\log(\frac{1}{\epsilon})}{k}$, $s_1 \leftarrow n\min\left\{1,\beta\right\} $
\For{phase $i \in\{1, \ldots, k\}$}
    \State $\mathcal{A}_i\leftarrow$ $s_i$ elements sampled from $\Omega \backslash S^{(i-1)}$
    \For{arm $a \in \mathcal{A}_i$}
        \State Play $S^{(i-1)} \cup\{a\}$ $m$ times
        \State Calculate the mean $\bar{f}\left(S^{(i-1)} \cup\{a\}\right)$
    \EndFor
    \State $a_i \leftarrow \arg \max _{a \in \mathcal{A}_i} \bar{f}\left(S^{(i-1)} \cup\{a\}\right)$
    \State $S^{(i)} \leftarrow S^{(i-1)} \cup\left\{a_i\right\}$
    \State $s_{i+1} \leftarrow (n - i + 1 )\min\{1,\beta\}$
\EndFor
\For{remaining time} 
    \State Play $S^{(k)}$
\EndFor
\end{algorithmic}
\end{algorithm}

The algorithm follows the explore-then-commit structure where base arms are added
to a super arm over time greedily until the cardinality constraint is satisfied and then exploits that super arm. However, in contrast to previous explore-then-commit approaches in \citep{nie2022explore,nie2023framework}, which at every exploration phase has a search space of $\mathcal{O}(n)$, to minimize its expected cumulative regret, SGB reduces its search space to $\mathcal{O}(\frac{n}{k} \min\{k,\log (\frac{1}{\epsilon})\})$ arms. The aim of reducing the searched arms in each iteration is to reduce the time spent in the exploration.

Let $S^{(i)}$ represent the super arm when $i<k$ base arms are selected. Our algorithm starts with the empty set, $S^{(0)}=\emptyset$. To add an arm to the set $S^{(i-1)}$, ETCG explores the full subset $\Omega\setminus S^{(i-1)}$. Instead, our procedure explores a smaller subset, i.e., a random subset $\mathcal{A}_i \subseteq \Omega\setminus S^{(i-1)}$. With $\beta = \log(\frac{1}{\epsilon})/k$, the cardinality of $\mathcal{A}_i$ is
\begin{equation}
    |\mathcal{A}_i| = s_i =  (n - i + 1 )\min\{1,\beta\}.
\end{equation} 

For $\beta < 1$, during each exploration phase $i$, while ETCG explores $(n - i + 1)$ arms, SGB only explores $(n - i + 1)\beta$ arms. Therefore, SGB requires fewer oracle queries per exploration phase than ETCG. For $\beta \geq 1$, during each exploration phase $i$, SGB explores $(n - i + 1)$ arms, leading it to become a deterministic greedy algorithm and recover the same results as ETCG. We note that $\beta<1$ happens when $\epsilon > e^{-k}$. Therefore, the lower bound on $\epsilon$ exponentially decreases as a function of $k$, ensuring this is true for most instances. To minimize the cumulative regret, $\epsilon$ is optimized as a function of $n$, $k$, and $T$.

Let $T_i$ denote the time step when  phase $i$ finishes, for $i \in \{1,\cdots, k\}$.  We also denote $T_0=0$ and $T_{k+1}=T$ for notational consistency. Let $\bar{f}_t(S)$ denote the empirical mean reward of set $S$ up to and including time $t$. Let \[\mathcal{S}_{i} = \{\  S^{(i-1)}\cup\{a\} : \ a\in \mathcal{A}_i, \mathcal{A}_i \subseteq \Omega\setminus S^{(i-1)}  \ \}\] denote the set of actions considered during phase $i$.  Each action comprises the super arm $S^{(i-1)}$ decided during the last phase and an additional base arm. Each action $S\in \mathcal{S}_{i}$ is played the same number of times; let $m$ denote that number. The choice of $m$ will be optimized later to minimize regret. At the end of phase $i\in\{1,\dots,k\}$, SGB selects the action that has the largest empirical mean,
\begin{align}
    a_{i} = \argmax_{a\in \mathcal{A}_i } \ \bar{f}_{T_{i}}(S^{(i-1)} \cup \{a\}), \label{eq:emp_best}
\end{align} and include it in the super arm $S^{(i)} = S^{(i-1)} \cup \{a_{i}\}$.  
During the final phase, the algorithm exploits $S^{(k)}$; it plays the same action $S_t = S^{(k)}$ for $t\in \{T_k+1,\cdots, T\}$. 

Similar to previous state-of-the-art approaches, SGB has low storage complexity. During exploitation,  for $t\in \{T_{k}+1,\cdots, T_{k+1}\}$, only the indices of the $k$ base arms are stored, and no additional computation is required. During exploration, for $t\in \{1,\cdots, T_{k}\}$, for every phase $i$, SGB needs to store the highest empirical mean and its associated base arm $a\in \mathcal{A}_i$. Therefore, SGB has $\mathcal{O}(k)$ storage complexity. In comparison, the algorithm suggested by \cite{streeter2008online} for the full-bandit adversarial environment has a storage complexity of $\mathcal{O}(nk)$.

We note that the reduction of exploration time through random subset sampling from the remaining arms comes at the expense of reduced offline approximation guarantee to $(1-1/e-\epsilon)$ in \citep{mirzasoleiman2015lazier}. Thus, it is apriori unclear if such an approach can maintain the online $(1-1/e)$-regret guarantees with the reduced exploration, which is studied in the next section.

\section{Regret Analysis}
\label{sec:regretanalysis}
This section analyses the regret for Algorithm \ref{alg:sgb}. We begin by stating the main theorem, which bounds the expected cumulative $(1-1/e)$-regret of SGB.

\begin{theorem}
\label{thm1}
For the decision-making problem defined in Section 2 with $T \geq n (k+1) \sqrt{\log(T)}$, the expected cumulative $(1-1 / e)$-regret of SGB is at most $\mathcal{O}(n^{\frac{1}{3}} k^{\frac{2}{3}} T^{\frac{2}{3}} \log(T)^{\frac{2}{3}})$.
\end{theorem}

The rest of the section provides the proof of this result. 

Since for each phase $i$, we select each action $S^{(i-1)} \cup\{a\} \in$ $\mathcal{S}_i$ exactly $m$ times, we consider the equal-sized confidence radii rad $=\sqrt{\log (T) / m}$ for all the actions $S^{(i-1)} \cup\{a\} \in \mathcal{S}_i$ at the end of phase $i$. Denote the event that the empirical means of actions played in phase $i$ are concentrated around their statistical means as
\begin{equation}
\label{emp:conce}
   \mathcal{E}_i=\bigcap_{S \cup\{a\} \in \mathcal{S}_i}\{|\bar{f}(S \cup\{a\})-f(S \cup\{a\})|<\operatorname{rad}\} \text {. } 
\end{equation}
Then we define the clean event $\mathcal{E}$ to be the event that the empirical means of all actions selected up to and including phase $k$ is within rad of their corresponding statistical means:
$$
\mathcal{E}=\mathcal{E}_1 \cap \cdots \cap \mathcal{E}_k .
$$
Although the $\mathcal{E}_i$ 's are not independent, by conditioning on the sequence of played subsets $\left\{S^{(0)}, S^{(1)}, \ldots, S^{(k)}\right\}$ and using the Hoeffding bound \citep{hoeffding1994probability}, we show in the Appendix that $\mathcal{E}$ happens with high probability. We use the concentration of empirical means, Equation (\ref{emp:conce}), and properties of submodularity to show the following result.

\begin{lemma}
\label{lemma1}
Under the clean event $\mathcal{E}$, for all $i \in$ $\{1,2, \cdots, k\}$, for all positive $\epsilon$,
$$
\begin{aligned}
f(S^{(i)})-f(S^{(i-1)})
\geq \frac{1-\epsilon}{k}(f(S^{\star})-f(S^{(i-1)}))-2 \operatorname{rad}.
\end{aligned}
$$
\end{lemma}
\begin{proof}
Recall that $a_{i}$, defined in \eqref{eq:emp_best}, is the index of the arm that with $S^{(i-1)}$ forms the action with the highest empirical mean at the end of phase $i$, and $S^{(i)}=S^{(i-1)} \cup \{a_{i}\}$. Let $a_{i}^*$ denote the index of the arm that with $S^{(i-1)}$ forms the action with the highest expected value. For each $a \in \mathcal{A}_i$, the event that the empirical mean $\bar{f}(S^{(i-1)} \cup \{a\})$ is concentrated within a radius of size $\mathrm{rad}$ around the expected value. We lower bound the expected reward $f(S^{(i)})$ for the empirically best action in phase $i$, $S^{(i)}= \{a_{i}\}\cup S^{(i-1)}$.   To do so, we apply \eqref{emp:conce} to two specific arms, the empirically best $a_{i}$ out of $\mathcal{A}_i$ and the statistically best $a_{i}^\star$ out of $\mathcal{A}_i$. 
\begin{align}
    f(S^{(i)}) &= f(S^{(i-1)} \cup \{a_{i}\}) \tag{by design}\\
    &\geq \bar{f}(S^{(i-1)} \cup \{a_{i}\}) - \mathrm{rad} \tag{ using  \eqref{emp:conce} } \\
    &\geq \bar{f}(S^{(i-1)} \cup \{a_{i}^\star\}) - \mathrm{rad} \tag{$a_{i}$ has the highest empirical mean} \\
    &\geq  f(S^{(i-1)} \cup \{a_{i}^\star\}) - 2\mathrm{rad}. \tag{ using  \eqref{emp:conce} } 
\end{align}
Furthermore, using Lemma 2 in \citep{mirzasoleiman2015lazier}, with $s_i =  (n - i + 1)\min\{1,\frac{\log(\frac{1}{\epsilon})}{k}\}$, we have 
\begin{equation}
    f(S^{(i-1)} \cup \{a_{i}^\star\}) - f(S^{(i-1)}) \geq \frac{1-\epsilon}{k} (f(S^{\star}) - f(S^{(i-1)}).
\end{equation}
Combining the above results, we conclude that
\begin{equation*}
\begin{aligned}
    f(S^{(i)}) - f(S^{(i-1)}) &\geq  \frac{1-\epsilon}{k} (f(S^{\star}) - f(S^{(i-1)}) - 2\mathrm{rad}.
\end{aligned}    
\end{equation*}   
\end{proof}

This lemma identifies a lower bound of the expected marginal gain $f\left(S^{(i)}\right)-f\left(S^{(i-1)}\right)$ of the empirically best action $S^{(i)}$ at the end of phase $i$. As a corollary of Lemma \ref{lemma1}, using properties of submodular set functions, we can lower bound the expected value of SGB's chosen set $S^{(k)}$ of size $k$, which is used for exploitation;

\begin{corollary}
\label{cor:main}
Under the clean event $\mathcal{E}$, for all positive $\epsilon$,  
$$
f\left(S^{(k)}\right) \geq (1 - \frac{1}{e})f(S^{\star}) -(\epsilon  + 2k \mathrm{rad} ) .
$$
\end{corollary}
\begin{proof}
 Applying Lemma \ref{lemma1} result recursively for $i=k$, until we get to $S^{(0)}=\emptyset$; $f(\emptyset)=0$,
\begin{align}
    f(S^{(k)}) 
    & \geq \left[\frac{1-\epsilon}{k}f(S^{\star})-2 \mathrm{rad}\right]\sum_{\ell=0}^{k-1}(1-\frac{1-\epsilon}{k})^\ell.
    \label{eq:prf:mainthm:case1:70}
\end{align}
Simplifying the geometric summation,
\begin{align}
    \sum_{\ell=0}^{k-1}(1-\frac{1}{k})^\ell
    &= \frac{1-(1-\frac{1-\epsilon}{k})^k}{1-(1-\frac{1-\epsilon}{k})} \nonumber\\
    &= k \left(1-\left(1-\frac{1-\epsilon}{k}\right)^k\right). \nonumber  
\end{align}

Continuing with \eqref{eq:prf:mainthm:case1:70}, 
\begin{align}
    f(S^{(k)}) 
    & \geq \left[\frac{1-\epsilon}{k}f(S^{\star})-2 \mathrm{rad}\right]k \left(1-\left(1-\frac{1-\epsilon}{k}\right)^k\right) \nonumber\\
    &\geq \left(1-\left(1-\frac{1-\epsilon}{k}\right)^k\right)f(S^{\star}) - 2k \mathrm{rad}
    \tag{simplifying with $(1-\frac{1}{k})^k\leq 1$}\nonumber\\
    &\geq \left(1-e^{-(1-\epsilon)}\right)f(S^{\star}) - 2k \mathrm{rad}.\nonumber
\end{align}

Therefore, for $0\le \epsilon \leq 1$, using $e^{\epsilon} \leq 1+ e\epsilon$ , we have
\begin{align}
    f(S^{(k)}) &\geq (1 - \frac{1}{e}-\epsilon)f(S^{\star}) - 2k \mathrm{rad} \\
    &= (1 - \frac{1}{e})f(S^{\star}) -\epsilon f(S^{\star}) - 2k \mathrm{rad} 
    \tag{rearranging}\\
    &\geq (1 - \frac{1}{e})f(S^{\star}) -\epsilon  - 2k \mathrm{rad} 
    \tag{$f(S^{\star})\leq 1$}\\
    &= (1 - \frac{1}{e})f(S^{\star}) -(\epsilon  + 2k \mathrm{rad} ). \nonumber
\end{align}
\end{proof}

We use the above Corollary \ref{cor:main} to bound the expected cumulative regret of our proposed algorithm. We split the expected $(1-1/e)$-regret \eqref{eq:reg:exp1e} conditioned on the clean event $\mathcal{E}$ into two parts, one for the exploration and one for the exploitation, 
\begin{align}
    \hspace{-.3cm}\mathbb{E}[\mathcal{R}(T)|\mathcal{E}] &=\sum_{t=1}^T \left((1-\frac{1}{e})f(S^{\star})-\mathbb{E}[f(S_t)]\right) \nonumber\\
      &=\underbrace{\sum_{i=1}^{k} \sum_{t=T_{i-1}+1}^{T_{i}} \left((1-\frac{1}{e})f(S^{\star})-\mathbb{E}[f(S_t)]\right)}_{\text{Exploration}}  \nonumber \\
      &+\underbrace{\sum_{t=T_k+1}^T \left((1-\frac{1}{e})f(S^{\star})-\mathbb{E}[f(S^{(k)})]\right)}_{\text{Exploitation}}. \label{eq:regr:clean:twopart}
\end{align}
Note that during phase $i$, each of the $s_i$ actions in $\mathcal{S}_{i}$ is selected exactly $m$ times, thus $T_{i}-T_{i-1} = ms_i$. For each action $S_t$ choosed during phase $i$, that is for $t\in \{T_{i-1}+1, \cdots, T_{i}\}$, since $S^{(i-1)} \subset S_t$, by monotonicity of the expected reward function $f$ we have $f(S^{(i-1)}) \leq f(S_t)$. Thus we can upper bound the expected regret $\mathbb{E}[\mathcal{R}(T)|\mathcal{E}] $ incurred during the first $k$ phases (first term of \eqref{eq:regr:clean:twopart}) as 
\begin{align}
     &\sum_{i=1}^{k} \sum_{t=T_{i-1}+1}^{T_{i}} \left((1-\frac{1}{e})f(S^{\star})-\mathbb{E}[f(S_t)]\right) \nonumber \\
     &\leq \sum_{i=1}^{k}ms_i\left((1-\frac{1}{e})f(S^{\star})-\mathbb{E}[f(S^{(i-1)})]\right)  \nonumber\\
      &\leq ms_1\sum_{i=1}^{k}\left((1-\frac{1}{e})f(S^{\star})-\mathbb{E}[f(S^{(i-1)})]\right ) \nonumber\\
      &\leq m s_1 k.
     \label{eq:decomp1}
\end{align}
The last inequality follows because the rewards are in $[0,1]$.

We can upper bound the expected regret $\mathbb{E}[\mathcal{R}(T)|\mathcal{E}] $ incurred during the exploitation phase (phase $k+1$;  second term of \eqref{eq:regr:clean:twopart}) by applying Corollary \ref{cor:main} as follows
\begin{align}
    &\sum_{t=T_k+1}^T \left((1-\frac{1}{e})f(S^{\star})-\mathbb{E}[f(S^{(k)})]\right) \nonumber \\ 
    &\leq \sum_{t=T_k+1}^T(\epsilon+2k\mathrm{rad} ) \nonumber\\
    &\leq T\epsilon + 2kT\mathrm{rad}. \label{eq:14}
\end{align}
 Combining the upper bounds \eqref{eq:decomp1} and \eqref{eq:14}, we get 
\begin{align}
    \hspace{-.5cm}\mathbb{E}[\mathcal{R}(T)|\mathcal{E}] &\leq  ms_1k + \epsilon T+2kT\mathrm{rad} . 
\end{align}

We have $s_1 = n\min\{1,\frac{\log(\frac{1}{\epsilon})}{k}\} $ and rad $=\sqrt{\frac{\log (T)}{m}}$. Therefore, we have
\begin{align}
    \hspace{-.1cm}\mathbb{E}[\mathcal{R}(T)|\mathcal{E}] &\leq  mn\min\{k,\log(\frac{1}{\epsilon})\} +  \epsilon T+ 2kT\sqrt{\frac{\log (T)}{m}}.
\end{align}

First, trivially $\min\{k,\log(\frac{1}{\epsilon})\} \leq \log(\frac{1}{\epsilon})$, thus
\begin{align}
    \mathbb{E}[\mathcal{R}(T)|\mathcal{E}] \leq mn\log(\frac{1}{\epsilon}) +  \epsilon T+ 2kT\sqrt{\frac{\log (T)}{m}}. 
\end{align}
Setting the derivative (with respect to $\epsilon$) of the bound  to $0$,
\begin{align}
    0 = -mn\frac{1}{\epsilon}  +T + 0 \qquad \Rightarrow \qquad \epsilon = \frac{mn}{T}.
\end{align} The second derivative (with respect to $\epsilon$), $mn\epsilon^{-2},$  is positive so the stationary point is a minimizer.  

Plugging $\epsilon = \frac{mn}{T}$ into the regret upper bound,
\begin{align}
    \mathbb{E}[\mathcal{R}(T)|\mathcal{E}] %
    &\leq mn\log(\frac{T}{mn}) +  \frac{mn}{T} T+ 2kT\sqrt{\frac{\log (T)}{m}} \nonumber\\
    &\leq mn\log(T) +  mn + 2kT\sqrt{\frac{\log (T)}{m}} \nonumber \\
    &\leq 2mn\log(T)  + 2kT\sqrt{\frac{\log (T)}{m}} .
\end{align}
The above inequality is valid for all $m$ strictly greater than zero. Hence, to find a tighter bound, we find $m^{\star}$ that minimizes the right side. Thus we get
\begin{align}
    m^{\star} %
    &=\left(\frac{k T \sqrt{\log (T)}}{2n \log (T)}\right)^{\frac{2}{3}} =\left(\frac{kT }{2n \sqrt{\log (T)}}\right)^{\frac{2}{3}}. 
\end{align}
For $T \geq n(k+1)\sqrt{\log(T)}$, we have $m^\star \geq \frac{1}{2}$, therefore $\lceil m^\star \rceil \leq 2m^\star$. Plugging $m = \lceil m^\star \rceil$ into the regret bound,
\begin{align}
    \mathbb{E}[\mathcal{R}(T)|\mathcal{E}] %
    &\leq \lceil m^{\star} \rceil n\log(T) + 2kT\log (T)^{1/2} \lceil m^{\star} \rceil^{-1/2} \nonumber\\
    &\leq 2 m^{\star} n\log(T) + 2kT\log (T)^{1/2}  (m^{\star})^{-1/2} \nonumber\\
    &= 2^{\frac{1}{3}} k^{\frac{2}{3}} T^{2 / 3} n^{-2 / 3}  \log(T) ^{-1 / 3} n\log(T) \nonumber \\
    &+ 2^{\frac{4}{3}} k^{\frac{2}{3}}T\log (T)^{1/2} T^{-1 / 3} n^{1 / 3}  \log(T) ^{1 / 6} \nonumber\\
    &= 2^{\frac{1}{3}} k^{\frac{2}{3}} T^{2 / 3} n^{1 / 3}  \log(T) ^{2 / 3}  \nonumber \\
    &+ 2^{\frac{4}{3}} k^{\frac{2}{3}} T^{2 / 3}  n^{1 / 3}  \log(T) ^{2 / 3} \nonumber\\
    &\leq \mathcal{O}(n^{\frac{1}{3}} k^{\frac{2}{3}} T^{\frac{2}{3}} \log(T)^{\frac{2}{3}}  ).
\end{align}
Based on $m^{\star}$, we define the optimal $\epsilon^{\star}$ as follows
\begin{align}
    \epsilon^{\star} = \frac{m^{\star}n}{T} &= n T^{-1} \left(\frac{kT }{2n \sqrt{\log (T)}}\right)^{\frac{2}{3}}  = \left(\frac{nk^2 }{4 T \log (T)}\right)^{\frac{1}{3}}.\nonumber
\end{align}
Under the bad event, i.e., the complement $\bar{\mathcal{E}}$ of the good event $\mathcal{E}$, given that the rewards are bounded in $[0,1]$, it can be easily seen that $\mathbb{E}[\mathcal{R}(T) \mid \mathcal{\bar{E}}]  \leq T$. Moreover, by using the Hoeffding inequality \citep{hoeffding1994probability}, for $T \geq nk$, we have $\mathbb{P}(\mathcal{\bar{E}})  \leq \frac{2}{T}$, see Appendix. Therefore, we obtain $\mathbb{E}[\mathcal{R}(T)] \leq \mathcal{O}(n^{\frac{1}{3}} k^{\frac{2}{3}} T^{\frac{2}{3}} \log(T)^{\frac{2}{3}}).$

\begin{remark}
    The framework of \cite{nie2023framework} that adapts offline algorithms for combinatorial optimization problems with robustness guarantees to online settings via the explore-then-commit approach can be applied to the offline algorithm in \citep{mirzasoleiman2015lazier}. However, as this offline algorithm has $(1-1/e-\epsilon)$-approximation guarantee, such an approach will give a weaker $(1-1/e-\epsilon)$-regret guarantee rather than $(1-1/e)$-regret guarantee studied in this paper.
\end{remark}

\begin{remark}\label{rem:unknown_horizon} For unknown time horizon $T$, the geometric doubling trick can extend our result to an anytime algorithm. To initialize the algorithm, we choose $T_0$ to be large enough, then we choose a geometric succession $T_i=T_0 2^i$ for $i\in \{1,2,\cdots\}$, and run our algorithm during the time interval $T_{i+1}-T_i$ with a complete restart. 
From Theorem 4 in \citep{Besson2018WhatDT}, we can prove that the regret bound preserves the $T^{2/3}$ dependency with changes only in the constant factor. 
\end{remark}

\begin{remark}
For the scenario we study in this paper of combinatorial multi-armed bandit with submodular rewards in expectation and under full-bandit feedback, it is still unknown if $\tilde{\mathcal{O}}(T^{1/2})$ expected cumulative $(1-1/e)$-regret is possible (ignoring $n$ and $k$ dependence), and only $\tilde{\mathcal{O}}(T^{2/3})$ bounds have been shown in the literature; see Table \ref{tab:related-work}. 
\end{remark}

\section{Experiments on Online Social Influence Maximization}
\label{sec:experiments}

\begin{figure*}[t]%
    \centering
    \subfloat[]{\label{fig:im:a}{\includegraphics[width=0.26\linewidth]{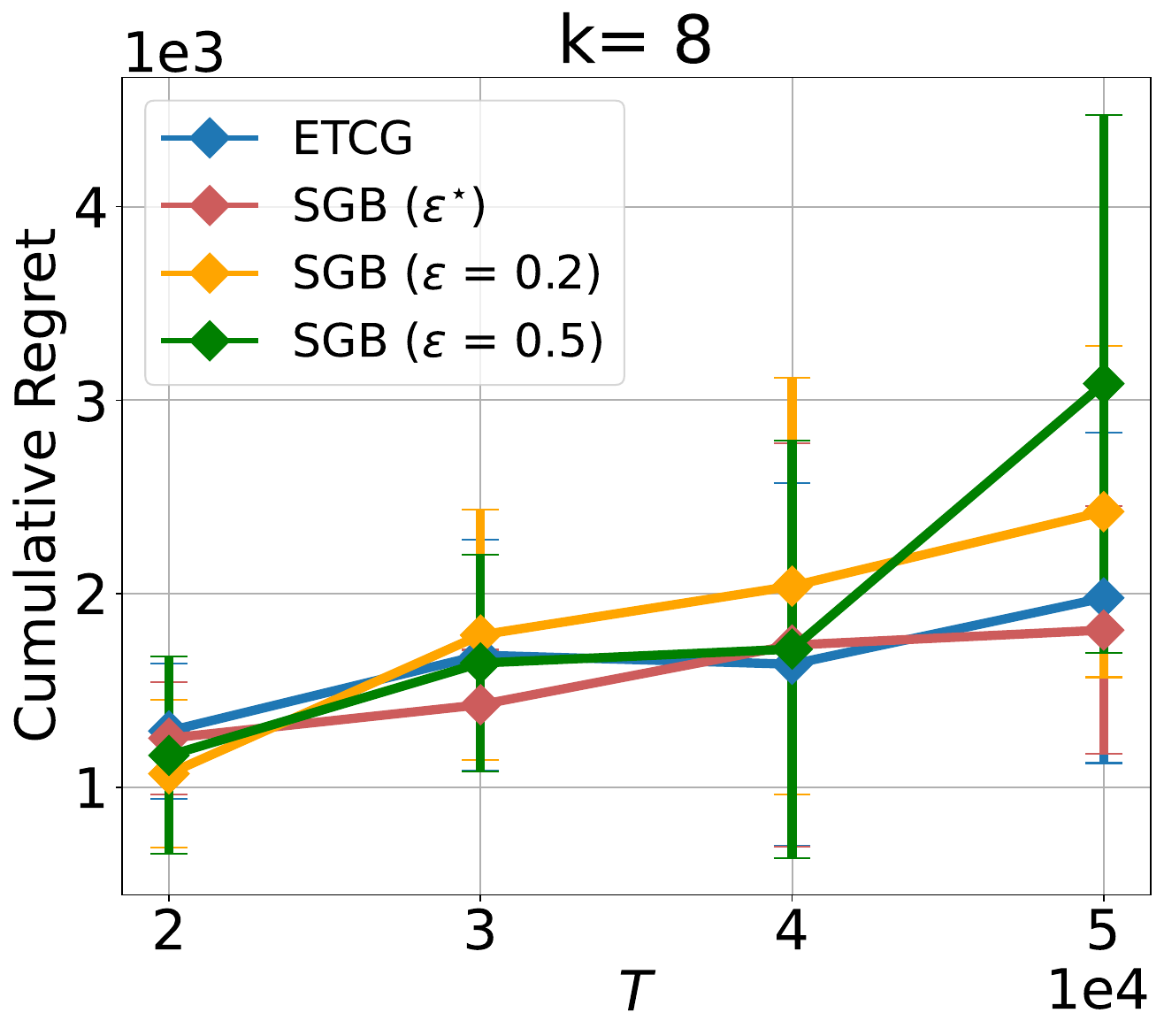} }} \hspace{1cm}%
    \subfloat[]{\label{fig:im:b}{\includegraphics[width=0.27\linewidth]{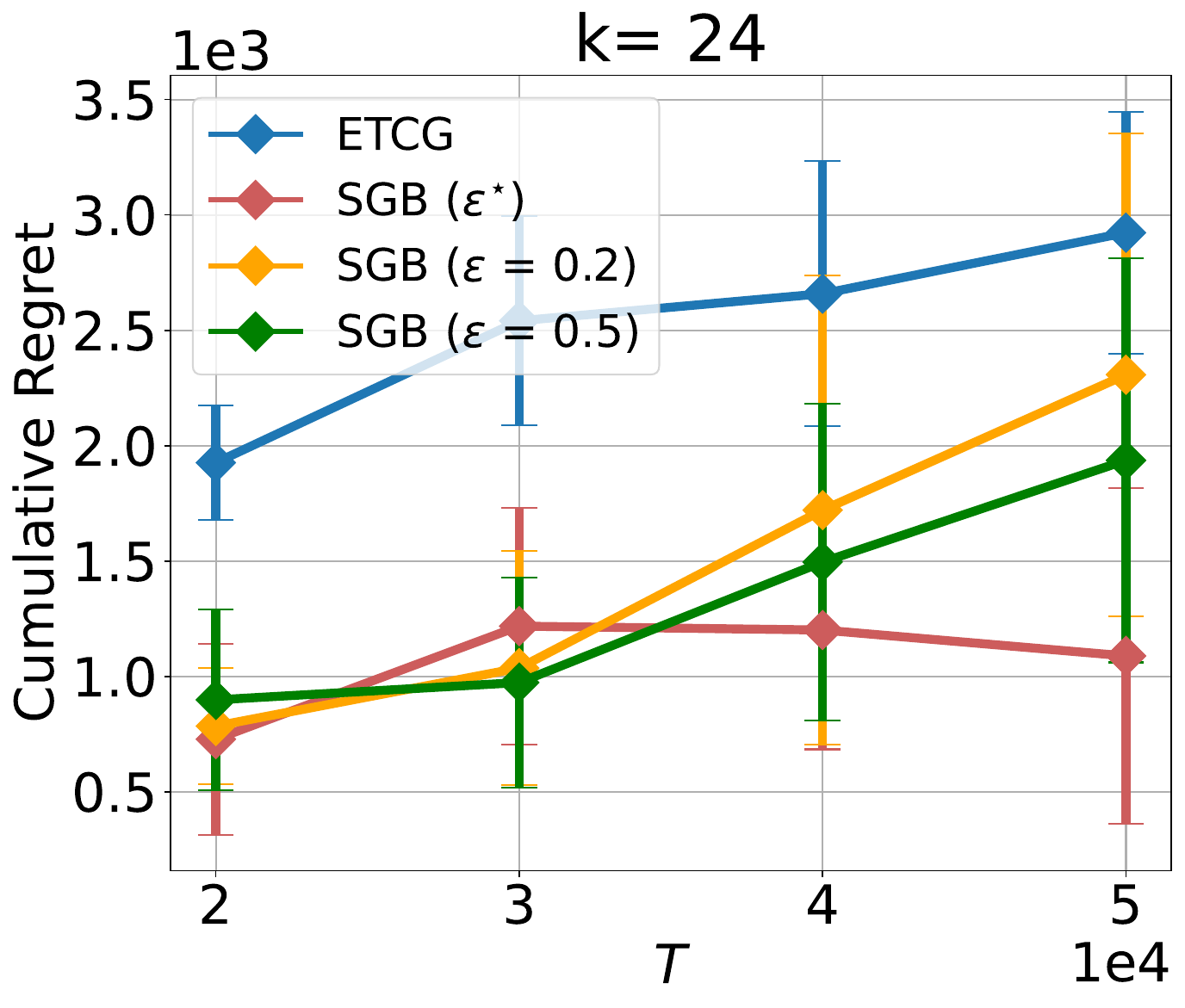} }} \hspace{1cm}%
    \subfloat[]{\label{fig:im:c}{\includegraphics[width=0.26\linewidth]{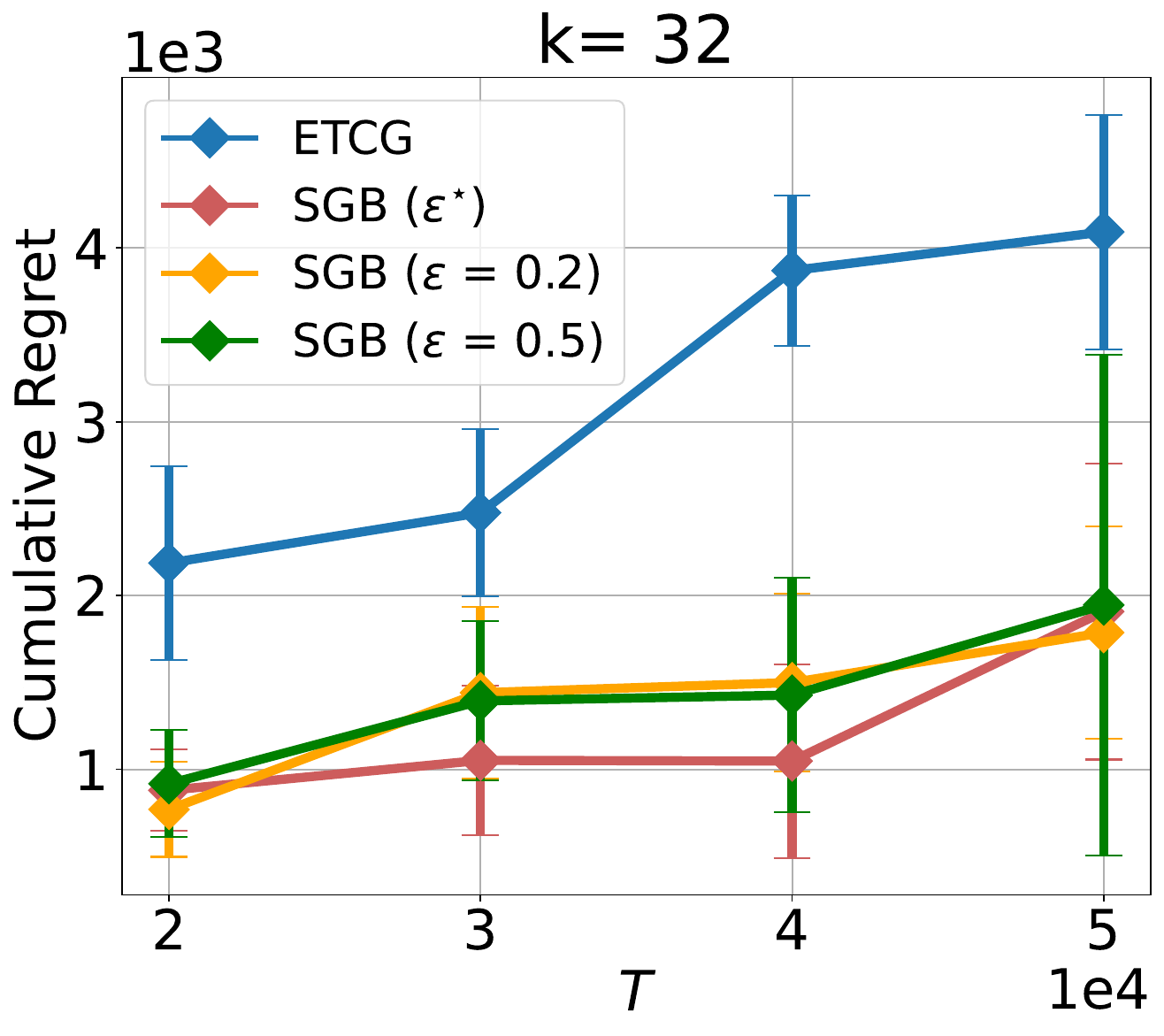} }}%
    \\
    \subfloat[]{\label{fig:im:d}{\includegraphics[width=0.26\linewidth]{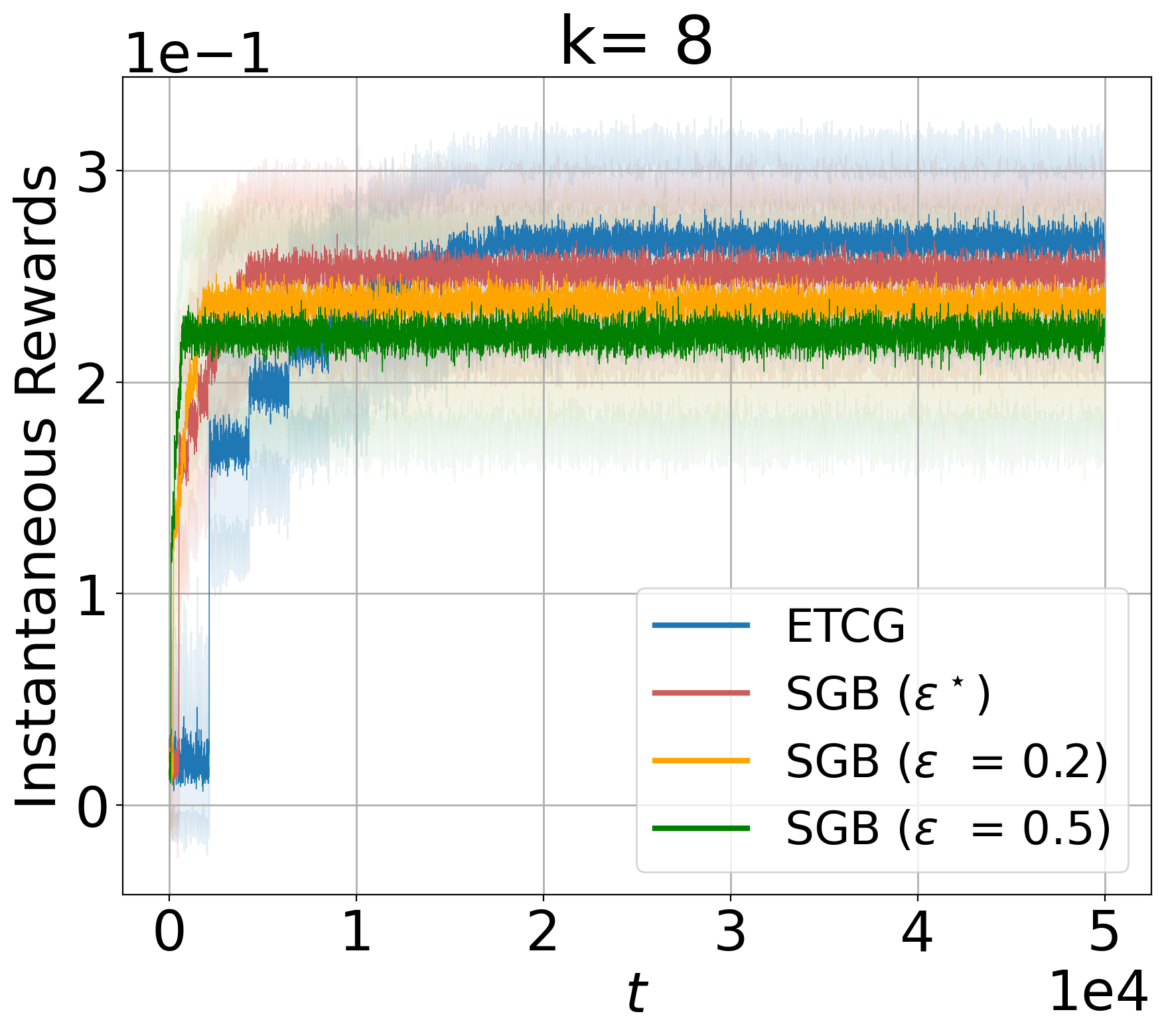} }} \hspace{1cm}%
    \subfloat[]{\label{fig:im:e}{\includegraphics[width=0.26\linewidth]{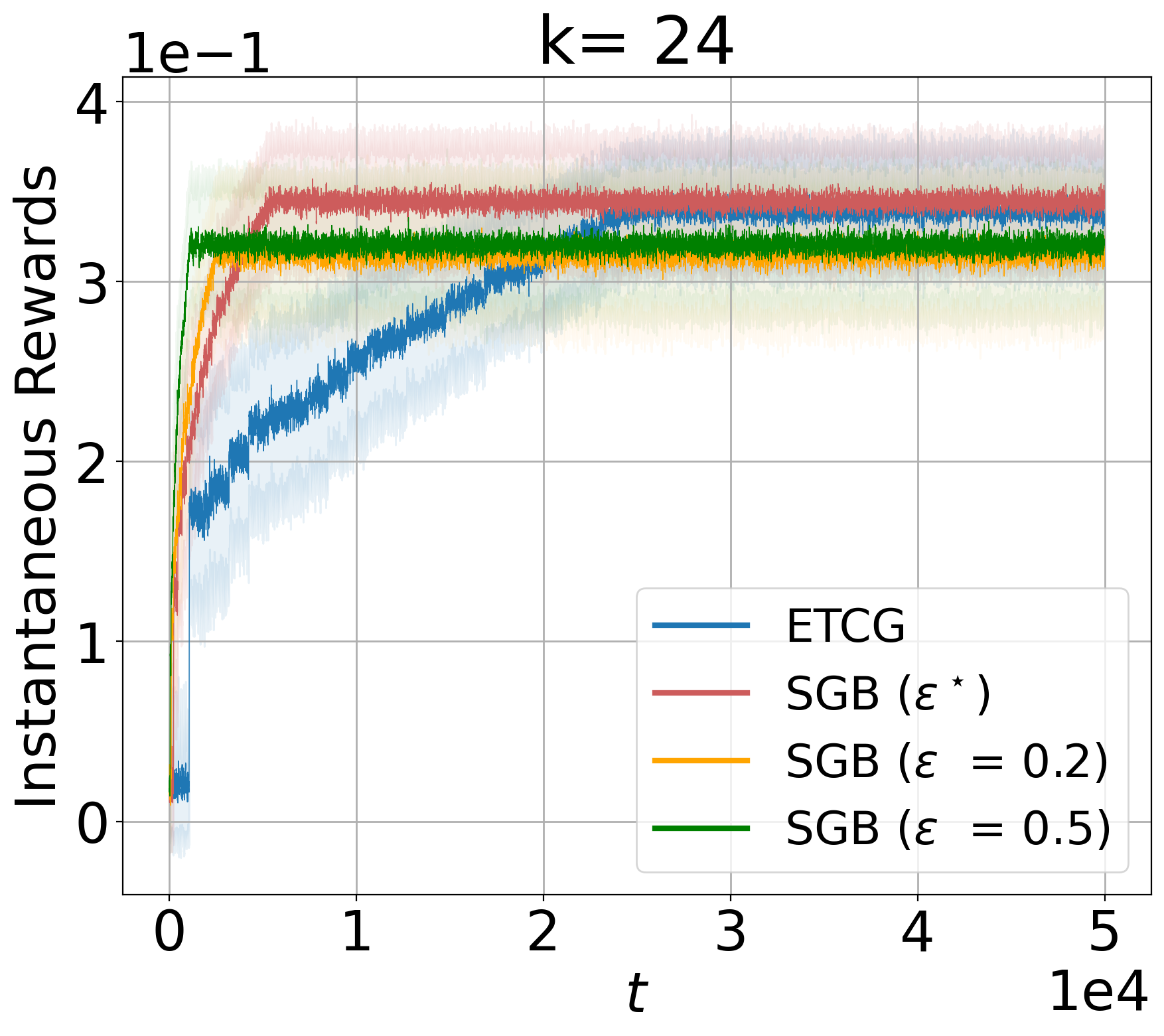} }}\hspace{1cm}%
    \subfloat[]{\label{fig:im:f}{\includegraphics[width=0.26\linewidth]{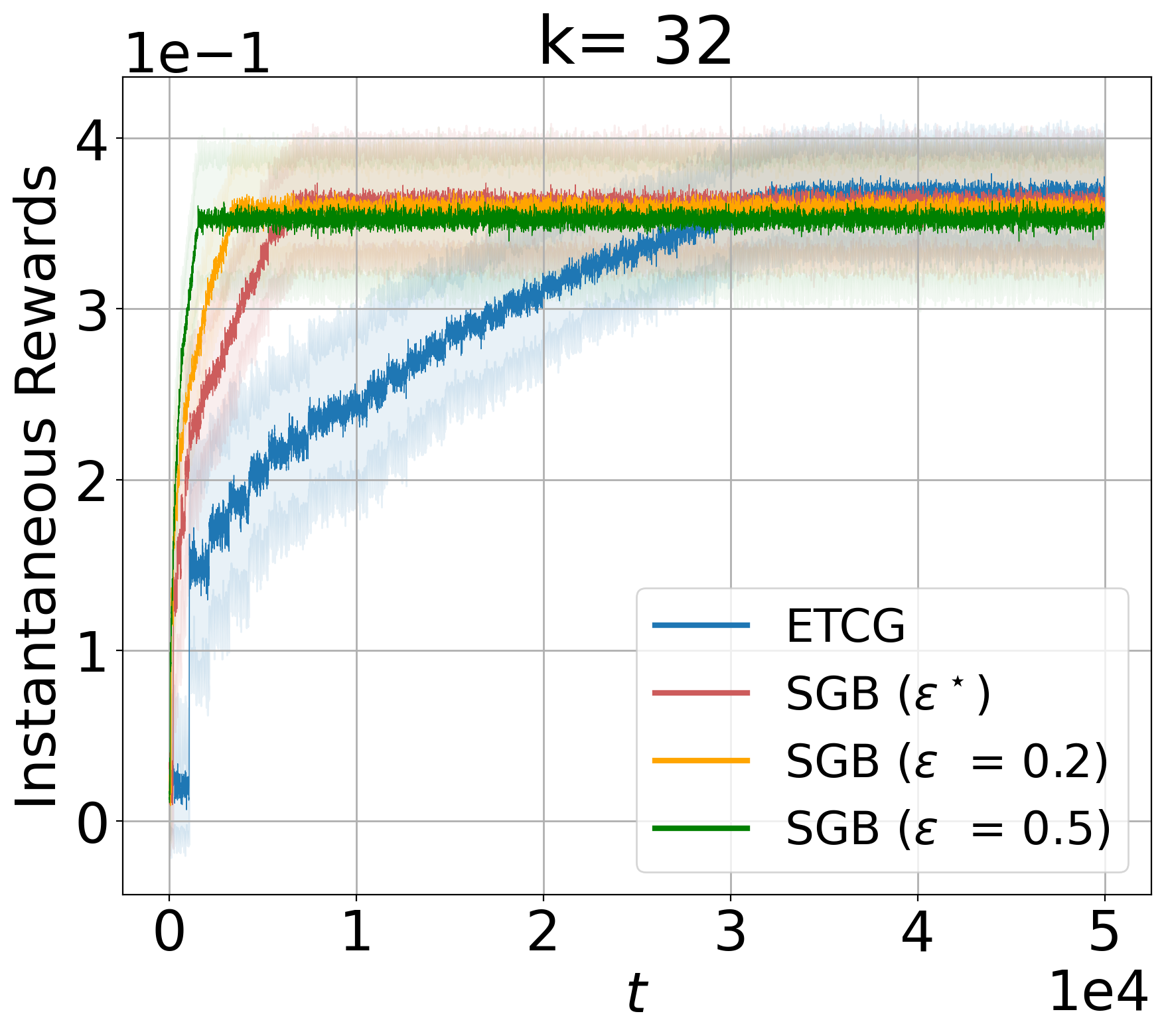} }}%
    \caption{Comparison of SGB, for different $\epsilon$ values, including $\epsilon^{\star} = (\frac{n k^2 }{4 T \log (T)})^{\frac{1}{3}}$, and ETCG. (a), (b), and (c) are the cumulative regret results as a function of horizon $T$. (d), (e), and (f) show the moving average plots of the immediate rewards as a function of t, with a window size of 100, with T fixed at $5\times10^4$, for which the respective $\epsilon^{\star}$ values are around 0.251, 0.522, and 0.632.} %
    \label{fig:im}%

\end{figure*}

\subsection{Problem Statement}

\textit{Social influence maximization} is a combinatorial problem, which consists of selecting a subset of nodes in a graph that can influence the remaining nodes. For instance, when marketing a newly developed product, one strategy is to identify a group of highly influential individuals and rely on their recommendations to reach a broader audience. Influence maximization can be formulated as a monotone submodular maximization problem, where adding more nodes to the selected set yields diminishing returns without negatively affecting other nodes. Typically, there is a fixed constraint on the cardinality of the selected set. While some works have addressed influence maximization as a multi-armed bandit problem with additional feedback \citep{lei2015online, wen2017online, vaswani2017model, li2020online, perrault2020budgeted}, this feedback is often unavailable in most social networks, except for a few public accounts. Recently, \cite{nie2022explore} proposed the ETCG algorithm for influence maximization under full-bandit feedback. Their algorithm demonstrated superior performance through empirical evaluations compared to other full-bandit algorithms. In this work, we compare our SGB method, for different $\epsilon$ values, including the optimized value $\epsilon^{\star} = (\frac{n k^2 }{4 T \log (T)})^{\frac{1}{3}}$, with ETCG \citep{nie2022explore, nie2023framework}. 

\subsection{Experiment Details}
For the experiments, instead of $(1-1/e)$ regret in Eq. (\ref{eq:reg:1e}), which requires knowing $S^{\star}$, we compare the cumulative rewards achieved by SGB for different $\epsilon$, including $\epsilon^\star$, and ETCG against $Tf(S^\mathrm{grd})$, where $S^\mathrm{grd}$ is the solution returned by the offline $(1-1/e)$-approximation algorithm suggested by \cite{nemhauser1978analysis}. Since $f(S^\mathrm{grd})\geq (1-1/e)f(S^{\star})$, thus $Tf(S^\mathrm{grd})$ is a more challenging reference value than $(1-1/e)Tf(S^{\star})$. 

We experimented using a portion of the Facebook network \citep{NIPS2012_7a614fd0}. We used the community detection method proposed by \citep{Blondel2008FastUO} to detect a community with 534 nodes and 8158 edges, enabling multiple experiments for various horizons. The diffusion process is simulated using the independent cascade model \citep{kempe2003maximizing}, wherein in each discrete step, an active node (that was inactive at the previous time step) independently tries to infect each of its inactive neighbors.  We used 0.1 uniform infection probabilities for each edge. For every time horizon $T\in\{2 \times 10^4, 3\times10^4, 4\times10^4, 5\times10^4\}$, we tested each method ten times.

\subsection{Experimental Results}

Figures (\ref{fig:im:a}), (\ref{fig:im:b}), and (\ref{fig:im:c}) show average cumulative regret curves for SGB with different values of the parameter $\epsilon$, including the optimal value $\epsilon^{\star} = (\frac{n k^2}{4 T \log(T)})^{\frac{1}{3}}$, for various time horizon values $T$, with a cardinality constraint $k$ set to 8, 24, and 32, respectively. The shaded areas depict the standard deviation. The figure axes are linearly scaled, so a linear cumulative regret curve corresponds to a linear $\widetilde{O}(T)$ cumulative regret. When $k=8$, SGB with $\epsilon^{\star}$ demonstrates nearly the lowest average cumulative regret across different time horizons $T$. However, with non-optimal values of $\epsilon$ (0.2 and 0.5), the cumulative regret of SGB is higher than that of ETCG. For higher values of $k$, such as 24 and 32, with $\epsilon^{\star}$ as shown in Figures (\ref{fig:im:b}) and (\ref{fig:im:c}), respectively, SGB with all the considered $\epsilon$ values outperforms ETCG with lower average cumulative regrets. 
Furthermore, Figures (\ref{fig:im:d}), (\ref{fig:im:e}), and (\ref{fig:im:f}) illustrate immediate rewards over a horizon $T=5 \times 10^4$ for cardinality constraints $k$ of 8, 24, and 32, and $\epsilon^{\star}$ values around 0.251, 0.522, and 0.632, respectively. The curves for all methods are smoothed using a moving average with a window size of 100. For $k = 8$, although ETCG finds a slightly better solution, SGB with all $\epsilon$ ends exploration much faster. For $k = 24$, as shown in Fig. (\ref{fig:im:e}), SGB using $\epsilon^{\star}$ ends exploration much faster than ETCG and achieves a better solution. Using other $\epsilon$ values ends exploration slightly faster than the optimal value but to a lower solution. Similarly, for $k = 32$, as shown in Fig. (\ref{fig:im:f}), SGB with different $\epsilon$ values ends exploration 30 times faster than ETCG to a solution within a $0.01$ neighborhood of 0.37. Furthermore, using $\epsilon^{\star}$ yields the best result compared to other values. 
Therefore, as predicted by the theory, SGB using $\epsilon^{\star}$ has lower expected cumulative regret than ETCG. Additionally, as observed in the experiments and predicted by the theory, our method becomes more effective for larger values of $k$.

\section{Conclusion}

This paper introduces SGB, a novel method in the online greedy strategy, which incorporates subset random sampling from the remaining arms in each greedy iteration. Theoretical analysis establishes that SGB achieves an expected cumulative $(1-1/e)$-regret of at most $\tilde{\mathcal{O}}(n^{\frac{1}{3}} k^{\frac{2}{3}} T^{\frac{2}{3}})$ for monotone stochastic submodular rewards, outperforming the previous state-of-the-art method by a factor of $k^{1/3}$ \citep{nie2023framework}. Empirical experiments on online influence maximization shows SGB's superior performance, highlighting its effectiveness and potential for real-world applications.

\bibliography{aaai24}


\onecolumn
\appendix

\section{Lemmas and Proofs} \label{prf:main}

\begin{lemma} [Hoeffding's inequality] \label{lem:hoeffding}
Let $X_1, \cdots, X_n$ be independent random variables bounded in the interval $[0, 1]$, and let $\bar{X}$ denote their empirical mean. Then we have for any $\gamma >0$,
\begin{align}
    \mathbb{P}\left( \big|\bar{X} -  \mathbb{E}[\bar{X}] \big| \geq \gamma  \right) \leq 2 \mathrm{exp} \left( - 2 n \gamma^2  \right).
\end{align}
\end{lemma}

\begin{lemma} \label{lem:probcleanevents}
The probability of the clean event $\mathcal{E}$, for $T \geq nk$, satisfies:
\begin{align}
    \mathbb{P}(\mathcal{E}) %
    & \geq 1 - \frac{2}{T}. \nonumber
\end{align}
\end{lemma}
\begin{proof}

We begin by breaking up the probability of the clean event $\mathcal{E}$ into conditional probabilities for the events $\{\mathcal{E}_{i}\}_{i=1}^{k}$ for each phase,
\begin{align}
    \mathbb{P}(\mathcal{E}) &= \mathbb{P}(\mathcal{E}_{1}\cap \dots \cap \mathcal{E}_{k}) \nonumber\\
    &= \prod_{i=1}^{k} \mathbb{P}(\mathcal{E}_{i}|\mathcal{E}_{1},\dots,\mathcal{E}_{i-1}). \label{eq:probbnd:Econd} 
\end{align}

Recall that  $\mathcal{E}_{i}$ is the event where the empirical means of all actions played in phase $i$ were concentrated around their statistical means. Which actions are available in phase $i$, namely  $\{S^{(i-1)}\cup\{a\}\}_{a\in \mathcal{R}\backslash S^{(i-1)}}$, depends on the action $S^{(i-1)}$ from the previous phase that had the highest empirical mean, which in turn is related to $\mathcal{E}_{i-1}$.  Although we cannot directly evaluate \eqref{eq:probbnd:Econd}, by conditioning on  $S^{(i-1)}$ we will be able to obtain a bound on \eqref{eq:probbnd:Econd}.

\begin{align}
    \mathbb{P}(\mathcal{E}_{i}|\mathcal{E}_{1},\dots,\mathcal{E}_{i-1}) %
    &= \sum_{S\in\left\{S' \ \big| \ S'\subseteq\Omega, \ |S'|=i-1\right\}} \mathbb{P}(S^{(i-1)}=S, \mathcal{E}_{i}|\mathcal{E}_{1},\dots,\mathcal{E}_{i-1}) \tag{law of total probability} \\
    &= \sum_{S\in\left\{S' \ \big| \ S'\subseteq\Omega, \ |S'|=i-1\right\}} \mathbb{P}(S^{(i-1)}=S|\mathcal{E}_{1},\dots,\mathcal{E}_{i-1}) \times \mathbb{P}(\mathcal{E}_{i}| S^{(i-1)}=S, \mathcal{E}_{1}, \dots, \mathcal{E}_{i-1}) \nonumber \\
    &= \sum_{S\in\left\{S' \ \big| \ S'\subseteq\Omega, \ |S'|=i-1\right\}} \mathbb{P}(S^{(i-1)}=S|\mathcal{E}_{1},\dots,\mathcal{E}_{i-1}) \times \mathbb{P}(\mathcal{E}_{i}| S^{(i-1)}=S), \label{eq:eventseqprob:fact:1}
\end{align} where \eqref{eq:eventseqprob:fact:1} follows from  rewards in phase $i$ being conditionally independent of rewards from other phases, given the corresponding actions played during phase $i$.

We now focus on bounding $\mathbb{P}(\mathcal{E}_{i}| S^{(i-1)}=S)$.  By conditioning on the set chosen in the previous phase, $S^{(i-1)}=S$, we know all the actions that will be played in the current phase $i$, $\{S^{(i-1)}\cup\{a\}\}_{a\in \mathcal{R}\backslash S^{(i-1)}}$.  The rewards of all the actions are bounded in $[0,1]$ and are conditionally independent (given the corresponding action).  

Apply Lemma \ref{lem:hoeffding} to  the empirical mean $\bar{f}(S^{(i-1)}\cup\{a\})$ of $m$ rewards for action $S^{(i-1)}\cup\{a\}$  and choosing $\epsilon=\mathrm{rad}=\sqrt{\log(T)/m}$ gives 

\begin{align}
    \mathbb{P}\left[\big|\bar{f}(S^{(i-1)}\cup\{a\})-f(S^{(i-1)}\cup\{a\}) \big| \geq \mathrm{rad} \right]     &\leq 2 \mathrm{exp} \left( - 2 m \mathrm{rad}^2  \right) \nonumber\\
    &= 2 \mathrm{exp} \left( - 2 m (\log(T)/m ) \right) \nonumber\\
    &= 2 \mathrm{exp} \left( - 2 \log(T)  \right) \nonumber\\
    &= \frac{2}{T^2}. \nonumber
\end{align}

Thus, for any individual action $S^{(i-1)}\cup\{a\} \in \mathcal{S}_{i}$, we can bound the probability that its sample mean $\bar{f}(S^{(i-1)}\cup\{a\})$ is within a specified confidence radius (complementary  of the event above) as
\begin{align}
    \mathbb{P}\left[\bigg|\bar{f}(S^{(i-1)}\cup\{a\})-f(S^{(i-1)}\cup\{a\}) \bigg| < \mathrm{rad} \right] 
    &= 1-\mathbb{P}\left[\bigg|\bar{f}(S^{(i-1)}\cup\{a\})-f(S^{(i-1)}\cup\{a\}) \bigg| \geq \mathrm{rad} \right] \nonumber\\
    &\geq 1-\frac{2}{T^2}. \label{eq:probbnd:single}
\end{align}

We can then use \eqref{eq:probbnd:single} to bound  $\mathbb{P}(\mathcal{E}_{i}| S^{(i-1)}=S)$ for any set $S\subset \Omega$ of $i-1$ arms.

\begin{align}
    \mathbb{P}(\mathcal{E}_{i}| S^{(i-1)}=S) %
    &=\mathbb{P}\left[ \bigcap_{a \in \mathcal{R} \backslash S^{(i-1)}}\left\{\bigg|\bar{f}(S^{(i-1)}\cup\{a\})-f(S^{(i-1)}\cup\{a\}) \bigg| < \mathrm{rad} \right\} \bigg |S^{(i-1)}=S \right] \tag{definition of $\mathcal{E}_{i}$}\\%
    &= \prod_{a \in \mathcal{R} \backslash S^{(i-1)} } \mathbb{P}\left[ \left\{\bigg|\bar{f}(S^{(i-1)}\cup\{a\})-f(S^{(i-1)}\cup\{a\}) \bigg| < \mathrm{rad} \right\}  \bigg |S^{(i-1)}=S \right] \tag{rewards are independent conditioned on actions} \\
    &\geq \left(1-\frac{2}{T^2} \right)^{| \mathcal{R} \backslash S^{(i-1)} |} \tag{using \eqref{eq:probbnd:single}}\\
    & = \left(1-\frac{2}{T^2} \right)^{n \min\{1,\frac{\log(\frac{1}{\epsilon})}{k}\}-i+1} \nonumber\\
    & \geq \left(1-\frac{2}{T^2} \right)^{n}.
    \label{eq:probbnd:phase}
\end{align}

Using \eqref{eq:eventseqprob:fact:1} and \eqref{eq:probbnd:phase}, we are now ready to lower bound the probability of a clean event. 

\begin{align}
    \mathbb{P}(\mathcal{E}) &= \mathbb{P}(\mathcal{E}_{1}\cap \dots \cap \mathcal{E}_{k}) \nonumber\\
    &= \prod_{i=1}^{k} \mathbb{P}(\mathcal{E}_{i}|\mathcal{E}_{1},\dots,\mathcal{E}_{i-1}) \nonumber \\
    &= \prod_{i=1}^{k} \sum_{S\in\left\{S' \ \big| \ S'\subseteq\Omega, \ |S'|=i-1\right\}} \mathbb{P}(S^{(i-1)}=S|\mathcal{E}_{1},\dots,\mathcal{E}_{i-1}) \times \mathbb{P}(\mathcal{E}_{i}| S^{(i-1)}=S) \tag{ using \eqref{eq:eventseqprob:fact:1} } \\
    &\geq \prod_{i=1}^{k} \sum_{S\in\left\{S' \ \big| \ S'\subseteq\Omega, \ |S'|=i-1\right\}} \mathbb{P}(S^{(i-1)}=S|\mathcal{E}_{1},\dots,\mathcal{E}_{i-1}) \times \left(1-\frac{2}{T^2} \right)^{n} \tag{ using \eqref{eq:probbnd:phase} } \\
    &= \prod_{i=1}^{k} \left(1-\frac{2}{T^2} \right)^{n} \sum_{S\in\left\{S' \ \big| \ S'\subseteq\Omega, \ |S'|=i-1\right\}} \mathbb{P}(S^{(i-1)}=S|\mathcal{E}_{1},\dots,\mathcal{E}_{i-1})  \nonumber \\ 
    &= \prod_{i=1}^{k} \left(1-\frac{2}{T^2} \right)^{n}   \nonumber \\ 
    & =   \left(1-\frac{2}{T^2} \right)^{nk} \nonumber\\
    &\geq   1-\frac{2nk}{T^2}. \tag{Bernoulli's inequality}
\end{align}
Therefore, for $T \geq nk$
\begin{align}
    \mathbb{P}(\mathcal{E}) \geq   1-\frac{2}{T}. 
\end{align}

This concludes the proof for Lemma \ref{lem:probcleanevents}.
\end{proof}

\end{document}